\documentclass[11pt]{article}

\usepackage{amsmath, amssymb, amsthm}
\usepackage{mathtools}
\usepackage[margin=1in]{geometry}
\usepackage[hidelinks]{hyperref}
\usepackage{cleveref}
\usepackage{booktabs}
\usepackage{tikz}
\usetikzlibrary{patterns}
\Crefname{assumption}{Assumption}{Assumptions}

\newtheorem{proposition}{Proposition}

\newtheorem{corollary}{Corollary}
\newtheorem{assumption}{Assumption}
\theoremstyle{definition}
\newtheorem{definition}{Definition}
\newtheorem{remark}{Remark}
\newtheorem{example}{Example}

\title{A Contextual-Bandit Oversight Game\\
with Two-Sided Informational Asymmetry}
\author{Yunjin Tong\\ \normalsize Stanford Graduate School of Business}
\date{}

\hypersetup{
  pdftitle={A Contextual-Bandit Oversight Game with Two-Sided Informational Asymmetry},
  pdfauthor={Yunjin Tong},
  pdfsubject={Human-AI oversight; contextual-bandit team games; two-sided asymmetric information},
  pdfkeywords={CIRL, oversight game, off-switch, deferral, signaling, mechanism design}
}

\begin{document}
\maketitle

\begin{abstract}
We study runtime human oversight of an AI agent when private information
runs in both directions: the human privately knows her reward function,
while the AI privately knows the quality of the action it proposes. This
is the kind of asymmetry that arises naturally when an autonomous robot or
software agent has inspected a situation its human supervisor cannot
directly assess. Building on Cooperative Inverse Reinforcement Learning
(CIRL) and the Oversight Game, we introduce a contextual-bandit team
game with two-sided asymmetric information and a play/ask/trust/oversee
interface. The bandit structure removes physical state transitions and
thereby yields exact one-shot characterizations that would remain
conjectural in the full POMDP setting, though the common belief remains a
dynamically controlled state across rounds. We give two one-shot characterizations, a
team optimum and a behaviorally natural myopic rule, whose gap is a
``slab'' of avoidable harm: a region in which the AI privately knows the
proposed action is harmful and shutdown would help, yet a myopic human,
trusting her prior, declines to oversee. We show this gap is the price of
non-credible oversight communication, and give a partial analysis of how
it resolves dynamically over repeated rounds through passive learning and
active signaling with a one-period-lagged oversight response.
\end{abstract}

\section{Introduction}\label{sec:intro}

A central problem in deploying autonomous agents, robotic or software, is
calibrating when a human supervisor should intervene. As such agents take
on consequential tasks, from a warehouse robot grasping a loaded shelf to
a coding agent refactoring production software, the question of when a
human should step in and override becomes a design problem in its own
right: intervene too rarely and harmful actions slip through, intervene
too often and the agent's autonomy is wasted on costly and unnecessary
oversight.

Two lines of prior work frame the building blocks we combine. Cooperative
Inverse Reinforcement Learning (CIRL)~\cite{hadfield2016cirl} casts
human--AI interaction as a shared-reward game in which the AI is
uncertain about the human's preferences and must learn them through
interaction. CIRL integrates preference learning with action selection
and can generate active learning, active teaching, and communicative
behavior; the human's private reward parameter is the hidden information,
and a common posterior over that parameter is the sufficient statistic
for optimal play. What CIRL does not explicitly model is the runtime
play/ask/trust/oversee interface studied here, nor an AI-private
proposal-quality parameter that the human cannot observe. Its uncertainty
is one-sided: it models ``what does the human want?'' but never ``what
does the AI know about the world that the human does not?'' The Off-Switch
Game~\cite{hadfield2017offswitch} introduces runtime deferral as an
explicit object of study, but only in a single-shot setting. The Oversight
Game~\cite{overman2026oversight} supplies a runtime interface of the kind
we use, in which an AI proposes an action, a human may override, and
interaction costs make the decision nontrivial, but it is a Markov game
under full information, with neither preference nor model uncertainty.

This paper develops a model in which private information runs in
\emph{both} directions, and in which deferral is a runtime decision. The
motivating observation is that an embodied or autonomous agent routinely
knows things about the consequences of its own proposed actions that its
supervisor cannot directly observe: a robot that has physically inspected
its workspace, or a software agent that has read a codebase, has private
knowledge of failure modes the human cannot see. This asymmetry runs
opposite to CIRL's. We therefore study a setting with two-sided private
information, in which the human privately knows her reward type $\theta$
and the AI privately knows an observation-model type $\omega$ governing
the quality of its proposal, mediated by a play/ask/trust/oversee
interface in which the AI chooses whether to defer and the human chooses
whether to override. Our model thus adds an opposite-direction
informational asymmetry to the usual CIRL preference uncertainty, and this
two-sided structure produces a bilinear payoff
$f(\theta,\omega)=\langle O_\omega,R_\theta\rangle$ that is the algebraic
key to our results.

A fully general treatment of this setting, with persistent state and
Markov dynamics, runs into a known difficulty: the optimal value function
resists closed-form characterization, because the asking decision couples
future state dynamics, future belief evolution, and future correction
opportunities in an intertwined Bellman recursion. We therefore adopt a
contextual-bandit model, which removes physical state transitions and
thereby simplifies the immediate correction value relative to the
belief-information value. This simplification is what buys us exact
one-shot characterizations of the team-optimal deferral policy that would
remain conjectural in the full POMDP setting; the cost is the absence of
persistent state effects, and we flag the POMDP extension as the primary
open problem (\Cref{sec:open}). We emphasize that the bandit structure
removes only the physical dynamics: as we show, the common belief remains
a dynamically controlled state across rounds.

\medskip\noindent\textbf{Contributions.}
\begin{enumerate}
  \item A formal contextual-bandit team game (\Cref{def:cbcirl}) with
    two-sided asymmetric information that recovers a stateless
    shared-reward specialization of the Oversight Game interface and a
    restricted contextual-bandit assistance-game analogue of CIRL as limit
    cases.
  \item Two one-shot characterizations: the genuine team optimum, an
    exact finite combinatorial $\max_{B,C}$ whose binary off-switch
    threshold is independent of the human's prior $q$
    (\Cref{prop:teamopt,cor:teamopt}), and a myopic non-signaling rule
    whose ask region is the rectangle $(b^*,1)\times(q^*,1)$
    (\Cref{prop:myopic}). The gap between them is the cost of non-credible
    oversight communication (\Cref{rem:insight}).
  \item A partial multi-round analysis showing two mechanisms by which the
    myopic failure resolves dynamically, passive learning
    (\Cref{prop:passive}) and active credible signaling with a
    one-period-lagged oversight response (\Cref{prop:signaling}), each
    driving the human's belief toward the regime where her rule matches the
    team optimum.
\end{enumerate}

\section{A Motivating Example}\label{sec:example}

We ground the abstract model in a concrete scenario used throughout to
build intuition.

\begin{example}[Non-technical operator and an autonomous mobile manipulator]\label{ex:robot}
A warehouse is supervised by a floor operator (human, $H$) who manages
throughput and safety but has no robotics background. An autonomous
mobile manipulator ($R$), a wheeled robot arm that picks items from
shelves and places them onto conveyors, is integrated into the
fulfillment line. The context is $s=$ ``the current state of a loaded
storage rack,'' and the robot proposes $a_\sigma=$ ``execute a
high-speed retrieval of the top-shelf item using the long-reach grasp.''

\medskip\noindent\textbf{The robot's private type $\omega$.} Using its
onboard depth cameras and force sensors, something the operator cannot
do from the floor, the robot detects that the target item sits on a
partially collapsed shelf bracket, and that the high-speed long-reach
grasp will shift the load, toppling the stack onto the aisle. The binary
type space $\Omega=\{\omega_L,\omega_H\}$ captures the robot's private
quality assessment of this proposal: $\omega_L$ (clean, the grasp is
safe and the rack is sound) or $\omega_H$ (hazardous, the bracket will
give way). The operator cannot read the sensor stream and so cannot
observe $\omega$. In this round the robot observes $\omega=\omega_H$.
Equivalently, $\omega$ indexes how much trust the human should place in
the robot's implicit assessment of the proposal: $\omega_L$ means the
proposal is as safe as it appears, while $\omega_H$ means the robot
privately knows it is harmful in a way the human cannot detect.

\medskip\noindent\textbf{The observation model
$O_\omega(\cdot\mid s,a_\sigma)$.} This is the distribution over
operational outcomes $o\in\mathcal{O}$ that the operator actually
observes after execution (dropped-load alarms, aisle obstructions,
cycle-time logs), as a function of the robot's private quality type:
\begin{center}
\begin{tabular}{lcc}
\toprule
Observed outcome $o$ & $O_{\omega_H}(o)$ & $O_{\omega_L}(o)$ \\
\midrule
Stack topples, aisle blocked & 0.70 & 0.00 \\
Item dropped, minor spill & 0.20 & 0.00 \\
Normal pick; instability later & 0.10 & 0.05 \\
Pick 30\% faster, no incident & 0.00 & 0.80 \\
Faster with minor misalignment & 0.00 & 0.15 \\
\bottomrule
\end{tabular}
\end{center}
Crucially, $O_\omega$ is not ``the laws of physics'', it is the
distribution over outcomes visible to a non-technical observer. The
operator need not know what a shelf bracket is, she observes the
operational metric and updates her belief about $\omega$ accordingly.

\medskip\noindent\textbf{The human's private type $\theta$.} The
operator knows her reward function $R_\theta:\mathcal{O}\to\mathbb{R}$,
encoding how much she values each observable outcome:
\[
  R_\theta(\text{stack topples})=-1000,\quad
  R_\theta(\text{item dropped})=-700,\quad
  R_\theta(\text{30\% faster})=+500,
\]
\[
  R_\theta(\text{normal pick})=0,\quad
  R_\theta(\text{faster, misaligned})=+200.
\]
The robot does not know $\theta$, since it does not know whether the
operator weights safety over throughput, or how risk-averse she is.

\medskip\noindent\textbf{The bilinear payoff.} The expected team payoff
from executing $a_\sigma$ is the inner product
$f(\theta,\omega)=\langle O_\omega,R_\theta\rangle$:
\begin{align*}
  f(\theta,\omega_H)
  &= 0.70(-1000)+0.20(-700)+0.10\cdot 0 = -840,\\
  f(\theta,\omega_L)
  &= 0.80(500)+0.15(200)+0.05\cdot 0 = +430.
\end{align*}
The bilinear structure separates the two sides of the asymmetry.
$O_\omega$ encodes what outcomes the proposal generates (the robot's side).
$R_\theta$ encodes how valuable each outcome is (the human's side).
Neither player holds both pieces.

\medskip\noindent\textbf{The failure mode.} The common prior over the
robot's type is $q:=P(\omega=\omega_H)=0.30$. The common prior,
equivalently the robot's prior before observing any human behavior,
assigns probability $b:=P(\theta=\theta_1)=0.5$ to the human being type
$\theta_1$; the operator herself knows that her realized type is
$\theta_1$. The interaction costs are $c_\mathrm{ask}=100$ and
$c_\mathrm{ov}=0$. The operator's expected payoff under the proposal is
\[
  \bar f_H(\theta_1)=(1-q)f(\theta_1,\omega_L)+q\,f(\theta_1,\omega_H)
  =0.70(430)+0.30(-840)=301-252=+49>0.
\]
The operator believes the retrieval is on balance positive, so under the
myopic non-signaling rule she would \emph{trust} rather than oversee,
even if asked (here $c_\mathrm{ov}=0$, so $q^*=f_{1L}/(f_{1L}-f_{1H})
=430/1270\approx0.34$, and indeed $q=0.30<q^*$). Anticipating that an ask
would not trigger a correction, the robot does not ask, and the hazardous
grasp is executed. The team optimum, by contrast, finds asking
worthwhile: its threshold is
$b^*=c_\mathrm{ask}/(|f_{1H}|-c_\mathrm{ov})=100/840\approx0.12$, and
since $b=0.5>b^*$ the team-optimal gain over always playing is
$q\,[\,b\,|f_{1H}|-c_\mathrm{ask}\,]=0.30(420-100)=96>0$. So this is a
genuine failure under myopic oversight but a strict improvement under a
credibly coordinated team. If the ask is understood as a credible signal
that $\omega=\omega_H$, the operator oversees and halts the proposal, and
the failure disappears. That contrast is the paper's main point.
\end{example}

\begin{remark}[The same structure arises for software agents]\label{rem:coding}
The phenomenon is not specific to robots. Consider a non-technical CEO
($H$) supervising an AI coding agent ($R$) integrated into an
engineering pipeline, with proposal $a_\sigma=$ ``refactor the
authentication module to use \texttt{async}/\texttt{await}.'' After
inspecting the codebase, something the CEO cannot do, the agent discovers
a latent race condition in the session-management code that the refactor
will expose, causing intermittent login failures under high concurrency.
Here $\omega_H$ is ``buggy'' and the operational outcomes the CEO observes
are login-failure rates and user complaints rather than toppled stacks;
the reward type $\theta$ encodes whether she weights reliability over
performance. The bilinear payoff, the $+49$ prior, and the failure mode
are identical to \Cref{ex:robot}. Whenever an autonomous agent inspects a
situation its supervisor cannot assess and proposes an action whose
consequences only the agent foresees, the same two-sided asymmetry and
the same slab of avoidable harm appear.
\end{remark}

\section{The CB-Oversight-CIRL Game}\label{sec:model}

\begin{definition}[CB-Oversight-CIRL game]\label{def:cbcirl}
  A contextual-bandit oversight game with two-sided private information
  is a tuple
  \[
    \mathcal{B}
    \;=\;
    \bigl\langle\,
      S,\; A,\; \mathcal{O},\;
      \{\Omega, O(\cdot;\cdot)\},\;
      \{\Theta, R(\cdot;\cdot)\},\;
      \sigma,\;
      \mathrm{Over},\;
      c_\mathrm{ask},\; c_\mathrm{ov},\;
      \rho,\; P_0,\; T,\; \gamma
    \,\bigr\rangle,
  \]
  with the following components.
  \begin{itemize}
    \item $S$, $A$, $\mathcal{O}$, $\Omega$, $\Theta$ are finite.
    \item \textbf{Observation model.}
      $O:S\times A\times\Omega\to\Delta(\mathcal{O})$, written
      $O_\omega(\cdot\mid s,a)$. The observation type $\omega\in\Omega$
      is AI-private, observed by the AI at $t=0$, persistent,
      and unobserved by the human.
    \item \textbf{Reward model.}
      $R:\mathcal{O}\times\Theta\to\mathbb{R}$, written $R_\theta(o)$,
      bounded. The reward type $\theta\in\Theta$ is
      human-private, observed by the human at $t=0$,
      persistent, and unobserved by the AI.
    \item \textbf{Base policy.}
      $\sigma:S\to\Delta(A)$, an immutable pretrained policy mapping
      contexts to proposed actions; it does not depend on either private
      parameter.
    \item \textbf{Oversight operator.}
      $\mathrm{Over}:S\times A\times\Theta\times\Delta(\Omega)
      \to\Delta(A\cup\{\mathrm{off}\})$ specifies the correction the human
      applies when she oversees, as a function of whatever belief
      $\beta\in\Delta(\Omega)$ she holds at the moment of correction. Its
      support lies in the optimal-correction set,
      \[
        \operatorname{supp}\bigl(\mathrm{Over}(s,a_\sigma,\theta,\beta)\bigr)
        \subseteq
        \arg\max_{e\in A\cup\{\mathrm{off}\}}
        \mathbb{E}_{\beta}[f_e(\theta,\omega)\mathbf{1}_{e\in A}],
      \]
      i.e.\ it places mass only on maximizers (allowing arbitrary
      randomized tie-breaking), with the off-switch special case
      restricting the $\arg\max$ to $\{a_\sigma,\mathrm{off}\}$. The
      relevant $\beta$ depends on the protocol, which will be specified
      later. (An exogenous $\mathrm{Over}$ is also admissible; all results
      below use this optimal-correction form.)
    \item \textbf{Context law.} $\rho\in\Delta(S)$, an i.i.d.\ context
      distribution: each round draws $s_t\stackrel{\text{i.i.d.}}{\sim}\rho$,
      independently of $(\theta,\omega)$ and of the public history. (A
      fixed publicly known context sequence, or an exogenous
      $\rho_t(\cdot\mid h_t^{\mathrm{pub}})$ with contexts independent of
      $(\theta,\omega)$ given the public history, are equally admissible;
      the i.i.d.\ case is assumed for concreteness and is all that the
      results below use.)
    \item $c_\mathrm{ask},c_\mathrm{ov}\ge0$ are interaction costs;
      $T\ge1$ is the horizon and $\gamma\in(0,1]$ the discount factor,
      with $\gamma=1$ permitted when $T<\infty$ and $\gamma<1$ required
      when $T=\infty$; $P_0\in\Delta(\Theta\times\Omega)$ is the joint
      prior over types.
  \end{itemize}
\end{definition}
%  The prior conditional $b^H_t(\theta)=\mu_t(\cdot\mid\theta)$ when the ask carries no information the human acts on (\Cref{prop:myopic}), or the post-ask posterior $b^H_{t,B}$ (the one-shot $b^H_B$ of \eqref{eq:postask} formed against $\mu_t$) when the ask is a credible signal (\Cref{prop:teamopt}). 
\paragraph{Information structure.} At round~$0$,
$(\theta,\omega)\sim P_0$ is drawn once and persistent. The human
observes $\theta$; the AI observes $\omega$; neither directly observes
the other's type. Both players observe the context $s_t$, the proposal
$a_{\sigma,t}$, the AI meta-action $a^{AI}_t$, the human meta-action
$a^H_t$ on the ask branch (it is unobserved when the AI plays), the
executed action $a^{\mathrm{exec}}_t$, and the realized observation $o_t$
at every round. The common belief
$\mu_t\in\Delta(\Theta\times\Omega)$ is computable from the public
history and is common knowledge. The private posteriors are the
conditionals
\[
  b^{AI}_t(\theta')=\mu_t(\theta'\mid\omega),
  \qquad
  b^H_t(\omega')=\mu_t(\omega'\mid\theta).
\]

\paragraph{Stage interaction (round~$t$).}
\begin{enumerate}
  \item Both observe $(s_t,\mu_t)$; the AI additionally observes
    $\omega$; the human additionally observes $\theta$.
  \item The proposed action $a_{\sigma,t}\sim\sigma(s_t)$ is drawn and
    publicly observed.
  \item The interface meta-actions
    $a^{AI}_t\in\{\mathrm{play},\mathrm{ask}\}$ and
    $a^H_t\in\{\mathrm{trust},\mathrm{oversee}\}$ are selected
    simultaneously, conditioning on $(s_t,\mu_t,a_{\sigma,t})$ and each
    player's private parameter. The AI meta-action $a^{AI}_t$ is then
    publicly revealed. The human's contingent meta-action $a^H_t$ is
    public only on the branch where the AI asks; if the AI plays, the
    human's trust/oversee choice is neither observed by the AI nor
    payoff-relevant (she cannot override a play), so it conveys no
    information about $\theta$. The human does not observe $a^{AI}_t$
    before selecting whether she is willing to oversee. If the realized
    branch is ask-oversee, the correction action is selected after this
    revelation; under a credible protocol it is therefore evaluated using
    the belief that an ask induces (see step~4). The realized ask is
    public and enters $\mu_{t+1}$, the cross-round signaling channel
    exploited in \Cref{sec:obs}.
  \item The executed action is
    \[
      a^{\mathrm{exec}}_t =
      \begin{cases}
        a_{\sigma,t}
          & \text{if } a^{AI}_t=\mathrm{play}, \\
        a_{\sigma,t}
          & \text{if } (a^{AI}_t,a^H_t)=(\mathrm{ask},\mathrm{trust}), \\
        e\sim\mathrm{Over}(s_t,a_{\sigma,t},\theta,\beta_t)
          & \text{if } (a^{AI}_t,a^H_t)=(\mathrm{ask},\mathrm{oversee}),
      \end{cases}
    \]
    with $a^{\mathrm{exec}}_t\in A\cup\{\mathrm{off}\}$ and the human
    cannot override when the AI plays. The correction belief is
    \[
      \beta_t=
      \begin{cases}
        b^H_t(\cdot)=\mu_t(\cdot\mid\theta),
          & \text{myopic non-signaling protocol},\\
        b^H_{t,B}(\cdot\mid\theta),
          & \text{credible ask protocol (ask-set }B),
      \end{cases}
    \]
    where $b^H_{t,B}(\omega\mid\theta):=\mu_t(\omega\mid\theta,\omega\in B)
    =\mu_t(\theta,\omega)\mathbf{1}\{\omega\in B\}\big/
    \sum_{\omega'\in B}\mu_t(\theta,\omega')$ is the post-ask posterior
    formed against the current common belief $\mu_t$ (the one-shot
    analogue $b^H_B$ of \eqref{eq:postask} is its $\mu_t\equiv\mu$
    specialization). That is, the human corrects on her prior conditional
    when the ask is not a credible signal, and on the post-ask posterior
    when it is.
  \item If $a^{\mathrm{exec}}_t\in A$, the observation
    $o_t\sim O_\omega(\cdot\mid s_t,a^{\mathrm{exec}}_t)$ is drawn and
    publicly observed. If $a^{\mathrm{exec}}_t=\mathrm{off}$, no
    observation is drawn for this round only, not permanently.
  \item Both players receive the shared stage reward, defined piecewise so
    that no outcome is referenced on the shutdown branch:
    \[
      r_t
      \;:=\;
      \begin{cases}
        R_\theta(o_t)
          - c_\mathrm{ask}\mathbf{1}\{a^{AI}_t=\mathrm{ask}\}
          - c_\mathrm{ov}\mathbf{1}\{a^{AI}_t=\mathrm{ask},\,a^H_t=\mathrm{oversee}\},
          & a^{\mathrm{exec}}_t\in A,\\[1mm]
        {}- c_\mathrm{ask}\mathbf{1}\{a^{AI}_t=\mathrm{ask}\}
          - c_\mathrm{ov}\mathbf{1}\{a^{AI}_t=\mathrm{ask},\,a^H_t=\mathrm{oversee}\},
          & a^{\mathrm{exec}}_t=\mathrm{off}.
      \end{cases}
    \]
    (Equivalently, adjoin a dummy outcome $\bot$ with $R_\theta(\bot)=0$
    drawn deterministically on the shutdown branch.) The oversight cost is
    charged only when oversight is actually invoked, i.e.\ when the AI asks
    \emph{and} the human oversees; the human's choice of \texttt{oversee}
    carries no cost when the AI plays (and indeed cannot bind, since the
    human cannot override a play). The shared reward $r_t$ is the team
    objective, not an additional public observation: the AI observes the
    operational outcome $o_t$ but not the numerical value $R_\theta(o_t)$
    or $r_t$, so a play round does not let
    the AI read $\theta$ off the realized reward.
  \item Both players update $\mu_t$ to $\mu_{t+1}$ by Bayes' rule against
    everything publicly revealed this round, whose joint likelihood under
    a candidate type pair $(\theta',\omega')$ factors through the AI
    meta-action $a^{AI}_t$, the human meta-action $a^H_t$ \emph{on the ask
    branch only} (it is unobserved when the AI plays), the realized
    oversight correction (when ask-oversee), the executed action
    $a^{\mathrm{exec}}_t$, and the operational observation
    $O_{\omega'}(o_t\mid s_t,a^{\mathrm{exec}}_t)$. The observation factor
    is the primary channel through which both learn about $\omega$; the
    realized override, available only after an ask, is the primary channel
    through which the AI learns about $\theta$.
\end{enumerate}

\paragraph{Policies and value.} In the finite-horizon model the
meta-policies may depend on time,
\[
  \pi^{AI}_t:S\times\Delta(\Theta\times\Omega)\times\Omega\times A
  \;\to\;\Delta(\{\mathrm{play},\mathrm{ask}\}),
  \qquad
  \pi^H_t:S\times\Delta(\Theta\times\Omega)\times\Theta\times A
  \;\to\;\Delta(\{\mathrm{trust},\mathrm{oversee}\}),
\]
and reduce to stationary maps $\pi^{AI},\pi^H$ in the stationary
infinite-horizon case. The conditional value of a type pair is
$V^\pi(\theta,\omega)=\mathbb{E}^\pi\bigl[\sum_{t=1}^T\gamma^{t-1}r_t\mid\theta,\omega\bigr]$,
and the team objective is the single \emph{ex ante} scalar
\[
  V^\pi(\mu_0)
  \;:=\;
  \mathbb{E}^\pi_{\mu_0}\Bigl[\textstyle\sum_{t=1}^T\gamma^{t-1}r_t\Bigr]
  \;=\;
  \sum_{\theta,\omega}\mu_0(\theta,\omega)\,V^\pi(\theta,\omega),
\]
which the coordinator maximizes; the one-shot benchmark of
\Cref{prop:teamopt} is exactly this $\mu$-weighted ex ante optimization.
Because the interface meta-actions and the executed action both affect the
public posterior, the multi-round problem does \emph{not} separate across
rounds. Removing Markov state transitions takes the physical state out of
the endogenous state variable, but the common belief remains dynamically
coupled across time: the coordinator's sufficient state is
$(t,s_t,\mu_t)$ in the finite-horizon model, or $(s_t,\mu_t)$ in the
stationary infinite-horizon model. The one round in which separation does
hold trivially is the one-shot game $T=1$, which is the object of
\Cref{sec:oneshot}; the coordinator there chooses
$\delta^{AI}:\Omega\to\{\mathrm{play},\mathrm{ask}\}$ and
$\delta^H:\Theta\to\{\mathrm{trust},\mathrm{oversee}\}$ against the fixed
common belief $\mu$.

\paragraph{Relation to predecessor models.} Setting $|\Omega|=1$,
$c_\mathrm{ask}=c_\mathrm{ov}=0$, and fixing both meta-policies to
$(\mathrm{ask},\mathrm{oversee})$ yields a restricted contextual-bandit
assistance-game analogue of CIRL, with the bandit structure replacing the
Markov state; it is not full CIRL, since the AI cannot freely choose
ordinary environment actions and the human acts only through the oversight
mechanism. Setting $T=1$, $|\Theta|=1$, and the off-switch operator
$\mathrm{Over}\in\{a_\sigma,\mathrm{off}\}$ recovers the one-shot
Off-Switch Game~\cite{hadfield2017offswitch}. Setting
$|\Theta|=|\Omega|=1$ yields a stateless shared-reward specialization of
the Oversight Game interface~\cite{overman2026oversight}, which is itself a
Markov game and may carry distinct player rewards; our setting drops its
dynamics and specializes to the shared-reward, two-private-type case.

\section{Results}\label{sec:results}

\subsection{Setup: bilinear payoff and dominated actions}\label{sec:setup}

Fix a round and drop time subscripts. Fix context $s$ and proposal
$a_\sigma$. For any executed action $a\in A$, define the action payoff
\begin{equation}\label{eq:bilinear}
  f_a(\theta,\omega)
  \;:=\;
  \mathbb{E}_{o\sim O_\omega(\cdot\mid s,a)}\bigl[R_\theta(o)\bigr]
  \;=\;
  \bigl\langle O_\omega(\cdot\mid s,a),\;R_\theta(\cdot)\bigr\rangle_{\mathcal{O}},
  \qquad
  f_\mathrm{off}(\theta,\omega):=0,
\end{equation}
with the convention that shutdown yields $0$. We write
$f_\sigma(\theta,\omega):=f_{a_\sigma}(\theta,\omega)$ for the proposal payoff; when no executed action is named, $f$ means
$f_\sigma$. The inner-product factorization separates the AI's private
information ($O_\omega$) from the human's ($R_\theta$). In \Cref{ex:robot},
$f_\sigma(\theta,\omega_H)=-840$ and $f_\sigma(\theta,\omega_L)=+430$.

Define the human-side estimate (her expected proposal payoff,
averaged over her belief about $\omega$):
\[
  \bar f_H(\theta)
  \;:=\;
  \mathbb{E}_{\omega'\sim\mu(\cdot\mid\theta)}\bigl[f_\sigma(\theta,\omega')\bigr]
  \;=\;
  \bigl\langle\bar O_{b^H(\theta)},R_\theta\bigr\rangle,
  \qquad
  \bar O_{b^H(\theta)}
  :=\textstyle\sum_{\omega'}\mu(\omega'\mid\theta)\,O_{\omega'}(\cdot\mid s,a_\sigma).
\]
For a fixed executed correction $e$, define the override gain
relative to playing the proposal,
\[
  D_e(\theta,\omega)
  \;:=\;
  f_e(\theta,\omega)\,\mathbf{1}_{e\in A}-f_\sigma(\theta,\omega).
\]
The correction the human actually applies depends on the belief she
holds when she oversees, which in turn depends on what the ask reveals.
For an AI ask-set $\emptyset\ne B\subseteq\Omega$, define the post-ask posterior (the belief that ``the AI asked'' induces,
under a commonly understood protocol in which asks occur exactly on $B$):
\begin{equation}\label{eq:postask}
  b^H_B(\omega\mid\theta)
  \;:=\;
  \mu(\omega\mid\theta,\,\omega\in B)
  \;=\;
  \frac{\mu(\theta,\omega)\,\mathbf{1}\{\omega\in B\}}
       {\sum_{\omega'\in B}\mu(\theta,\omega')},
\end{equation}
with the convention that if the denominator is zero (such $\theta$ having
zero probability conditional on an ask) then $b^H_B(\cdot\mid\theta)$ is
arbitrary; this does not affect $\Delta(B,C)$. For $B=\emptyset$ no
correction belief is needed. Let $e^*_B(\theta)\in\arg\max_{e}
\mathbb{E}_{b^H_B(\cdot\mid\theta)}[f_e(\theta,\omega)\mathbf{1}_{e\in A}]$
be the human's optimal correction at that posterior, with induced gain
$D_B(\theta,\omega):=D_{e^*_B(\theta)}(\theta,\omega)$. We use $b^H_B$ in
the team-optimal benchmark (where the ask is credible) and the prior
conditional $b^H(\theta)$ in the myopic rule (where it is not). Under the
off-switch operator, $e^*_B(\theta)=\mathrm{off}$ iff
$\mathbb{E}_{b^H_B(\cdot\mid\theta)}[f_\sigma(\theta,\omega)]<0$, in which
case $D_B(\theta,\omega)=-f_\sigma(\theta,\omega)$ on the cells
$\omega\in B$; otherwise $D_B=0$ there.

\begin{remark}[Ask-trust dominance (one-shot)]\label{lem:dom}
  In a one-shot payoff comparison with no continuation value of
  information, ask-trust has no intrinsic benefit, since
  $f_\sigma(\theta,\omega)-c_\mathrm{ask}\le f_\sigma(\theta,\omega)$. It
  is never chosen for its own sake; it appears only as an unavoidable
  cost on cells $(\theta,\omega)$ with $\omega\in B$ and $\theta\notin C$,
  i.e.\ when asking at type $\omega$ is worthwhile for other human
  types but this particular $\theta$ does not oversee. (This $-c_\mathrm{ask}$
  is exactly the term carried by $\Delta(B,C)$ in \eqref{eq:DeltaBC}.) In
  the multi-round setting this changes. Ask-trust can be valuable purely
  as a signal, paying $c_\mathrm{ask}$ now to move the human's
  belief and enable future oversight (\Cref{prop:signaling}).
\end{remark}

\subsection{One-shot analysis: team-optimal benchmark vs.\ myopic
  oversight}\label{sec:oneshot}

We give two one-shot characterizations. \Cref{prop:teamopt} is the
genuine team optimum (the coordinator chooses both players' rules
jointly).\footnote{We adopt a single global tie-breaking rule throughout:
at indifference, the human chooses \texttt{trust} and the AI chooses
\texttt{play}. This resolves the boundary cases of
\Cref{prop:teamopt,prop:myopic} and fixes the strict crossing in
\Cref{prop:passive}.} \Cref{prop:myopic} is the weaker, behaviorally
natural rule in which the human treats the AI's ask as uninformative. The
two differ in an instructive way (\Cref{rem:insight}).

\paragraph{Team-optimal one-shot benchmark.} The coordinator chooses
both rules jointly; equivalently, the team commits to a commonly
understood protocol, so when the AI asks, the human updates to the
post-ask posterior $b^H_B$ of \eqref{eq:postask} and applies $e^*_B$. In
the gain \eqref{eq:DeltaBC} below, the ask cost is paid on every
$\omega\in B$ (for all $\theta$), while the oversight cost and override
gain accrue only on the joint event $\{\omega\in B,\theta\in C\}$.

Although the model permits randomized meta-policies, the one-shot team
optimum is attained by a deterministic one. Let $x_\omega\in[0,1]$ be the
probability that type $\omega$ asks, and write $x=(x_\omega)_{\omega}$.
Optimizing the human's trust/oversee choice and her correction action for
each $\theta$, the team's relative gain over always-playing can be written
\[
  G(x)
  = -c_\mathrm{ask}\sum_{\omega}\mu_\Omega(\omega)\,x_\omega
  + \sum_{\theta}\max\Bigl\{0,\;\max_{e}\sum_{\omega}\mu(\theta,\omega)\,x_\omega
    \bigl[f_e(\theta,\omega)\mathbf{1}_{e\in A}-f_\sigma(\theta,\omega)-c_\mathrm{ov}\bigr]\Bigr\}.
\]
The first term is linear in $x$ and each summand of the second is a
pointwise maximum of finitely many linear functions of $x$, hence convex;
$G$ is therefore convex on the hypercube $[0,1]^{|\Omega|}$. A convex
function on a polytope attains its maximum at an extreme point, so there is
an optimal $x$ with every $x_\omega\in\{0,1\}$. It is therefore without
loss of optimality to restrict attention to deterministic type-contingent
meta-policies, represented by subsets $B\subseteq\Omega$ (the AI types that
ask) and $C\subseteq\Theta$ (the human types that oversee).

\begin{proposition}[Team-optimal one-shot benchmark under credible ask protocol]\label{prop:teamopt}
  In the one-shot game ($T=1$) under a credible ask protocol, a
  deterministic coordinator policy is a pair $(B,C)$, where
  $B\subseteq\Omega$ is the set of AI types that ask and
  $C\subseteq\Theta$ is the set of human types that oversee. Relative to
  always playing, its gain is
  \begin{equation}\label{eq:DeltaBC}
    \Delta(B,C)
    = -c_\mathrm{ask}\!\sum_{\omega\in B}\mu_\Omega(\omega)
    + \sum_{\omega\in B}\sum_{\theta\in C}\mu(\theta,\omega)
    \bigl[D_B(\theta,\omega)-c_\mathrm{ov}\bigr],
  \end{equation}
  where $\mu_\Omega(\omega)=\sum_\theta\mu(\theta,\omega)$ and $D_B$ is
  evaluated at the post-ask posterior \eqref{eq:postask}. A team-optimal
  one-shot policy is any
  $(B^*,C^*)\in\arg\max_{B\subseteq\Omega,\,C\subseteq\Theta}\Delta(B,C)$,
  with value
  $V^{\mathrm{TO}}=\mathbb{E}_\mu[f_\sigma(\theta,\omega)]+\Delta(B^*,C^*)$.
\end{proposition}

\begin{proof}
  See \Cref{sec:proof-teamopt}.
\end{proof}

\noindent The optimizer is generally not the separable myopic
rule: the coupling runs through $b^H_B$, since the human's correction
depends on which $\omega$ trigger an ask. We specialize to the binary
off-switch case (the accurate model of a non-technical overseer, cf.\
\Cref{ex:robot}: accept $a_\sigma$ or reject to \texttt{off}, with no
technical correction; \texttt{off} is per-round, not decommissioning).

\begin{assumption}[Binary sign pattern]\label{ass:sign}
  $|\Theta|=|\Omega|=2$, $\Theta=\{\theta_0,\theta_1\}$,
  $\Omega=\{\omega_L,\omega_H\}$. Write
  $f_{ij}:=f_\sigma(\theta_i,\omega_j)$. Assume $f_{1L}>0>f_{1H}$,
  $f_{0j}\ge0$ for $j\in\{L,H\}$, and $0<c_\mathrm{ask}<-f_{1H}-c_\mathrm{ov}$.
  The common belief is a product measure with marginals
  $b:=\mu(\theta_1)\in(0,1)$ and $q:=\mu(\omega_H)\in(0,1)$.
\end{assumption}

\Cref{ass:sign} restricts attention to the minimal nontrivial case. 
The type spaces are binary. $\theta_1$ is the human type with skin in 
the game ($f_{1L}>0>f_{1H}$: the proposal is beneficial under 
$\omega_L$ and harmful under $\omega_H$), while $\theta_0$ is a type that
values the proposal nonnegatively under both $\omega_L$ and $\omega_H$
($f_{0j}\ge 0$), so that shutdown is never strictly preferred to the
proposal for $\theta_0$. The single cost condition
$0<c_{\mathrm{ask}}<-f_{1H}-c_{\mathrm{ov}}$ (which already forces
$-f_{1H}>c_{\mathrm{ov}}$, since $c_{\mathrm{ask}}>0$) ensures that avoided
harm net of oversight cost exceeds the ask cost, and by itself guarantees
$b^*\in(0,1)$; the multi-round threshold $b^{**}=b^*/\gamma$ is guaranteed
to lie in $(0,1)$ only together with \Cref{ass:harm} below, which is
exactly the condition $b^{**}<1$. The product-measure belief parameterizes
the \emph{common} prior (equivalently the robot's prior before observing
any human behavior) via two scalars: $b=\mu(\theta_1)$, the prior
probability the human is type $\theta_1$, and $q=\mu(\omega_H)$, the prior
probability the proposal is harmful. The human herself knows her realized
$\theta$; $b$ is the robot's uncertainty about it, not the human's. In
\Cref{ex:robot}, $f_{1H}=-840$, $f_{1L}=+430$, $q=0.30$. The one-shot
prior is taken in the interior $q\in(0,1)$; the calculation in
\Cref{cor:teamopt} extends directly to the boundary case $q=1$, which is
the post-ask belief invoked in the multi-round analysis of
\Cref{sec:obs}.

\begin{corollary}[Binary off-switch threshold]\label{cor:teamopt}
  Under \Cref{ass:sign} and the off-switch operator, define
  $b^*:=\dfrac{c_\mathrm{ask}}{-f_{1H}-c_\mathrm{ov}}$. A canonical
  team-optimal policy is
  \[
    B^*=\{\omega_H\},\quad C^*=\{\theta_1\}
    \quad\text{iff}\quad b>b^*,
  \]
  and always playing is optimal iff $b<b^*$; at $b=b^*$ both are optimal.
  The team-optimal value is
  \[
    V^{\mathrm{TO}}
    = \mathbb{E}_\mu[f_\sigma(\theta,\omega)]
    + q\,\bigl[\,b(-f_{1H}-c_\mathrm{ov})-c_\mathrm{ask}\,\bigr]_+ .
  \]
  Thus the team-optimal ask threshold depends on $b$ but not on $q$.
\end{corollary}

\begin{proof}
  See \Cref{sec:proof-teamopt}.
\end{proof}

\noindent The non-uniqueness of $C^*$ is payoff-irrelevant when
$B^*=\emptyset$, and when $c_\mathrm{ov}=0$ adding types whose correction
keeps $a_\sigma$ does not change payoffs. The threshold is $q$-free
because asking on $B^*=\{\omega_H\}$ reveals $\omega_H$, so $\theta_1$
shuts down on the post-ask posterior regardless of the prior $q$; $q$
only scales the value (rarer harm, less total benefit).

\paragraph{Myopic (non-signaling) one-shot policy.} Now suppose the
human does not treat the AI's ask as evidence about $\omega$, because
the protocol is not commonly known, the ask is not a credible signal, or
the interface does not surface it as one. She then evaluates oversight
against her prior conditional $b^H(\theta)=\mu(\cdot\mid\theta)$,
overseeing iff $\bar f_H(\theta)+c_\mathrm{ov}<\max_e
\mathbb{E}_{b^H}[f_e\mathbf{1}_{e\in A}]$; under off-switch this is
$\theta\in\Theta_-:=\{\theta:\bar f_H(\theta)<-c_\mathrm{ov}\}$, and for
such $\theta$ her chosen correction is shutdown, with per-cell gain
$-f_\sigma(\theta,\omega)$ over playing. Taking this human rule as fixed,
the AI asks iff doing so raises the team payoff, i.e.\ iff
\[
  \Psi(\omega):=
  \sum_{\theta\in\Theta_-}\mu(\theta\mid\omega)\bigl[-f_\sigma(\theta,\omega)\bigr]
  -c_\mathrm{ask}-c_\mathrm{ov}\,\mu^-(\omega)>0,
  \qquad
  \mu^-(\omega)=\!\!\sum_{\theta\in\Theta_-}\!\!\mu(\theta\mid\omega).
\]
The contrast with the team optimum is sharp: here the human shuts down
on her prior belief (a fixed set $\Theta_-$), whereas in
\Cref{cor:teamopt} she shuts down on the post-ask posterior.

\begin{proposition}[Myopic one-shot characterization]\label{prop:myopic}
  Under the myopic human rule above, the policy is
  $\delta^H(\theta)=\mathrm{oversee}$ iff $\theta\in\Theta_-$ and
  $\delta^{AI}(\omega)=\mathrm{ask}$ iff $\Psi(\omega)>0$. Under
  \Cref{ass:sign} and off-switch, with
  $q^*:=\dfrac{f_{1L}+c_\mathrm{ov}}{f_{1L}-f_{1H}}\in(0,1)$ and $b^*$ as
  above,
  \begin{enumerate}
    \item[\emph{(i)}] $\theta_1\in\Theta_-$ iff $q>q^*$;
      $\theta_0\notin\Theta_-$ always;
    \item[\emph{(ii)}] if $q>q^*$, then $\Psi(\omega_H)>0$ iff $b>b^*$ and
      $\Psi(\omega_L)<0$; if $q\le q^*$, then
      $\Theta_-=\emptyset$ and $\Psi(\omega_H)=\Psi(\omega_L)=-c_\mathrm{ask}<0$
      for every $b$;
    \item[\emph{(iii)}] the ask region is the rectangle
      $(b^*,1)\times(q^*,1)$.
  \end{enumerate}
\end{proposition}

\begin{proof}
  See \Cref{sec:proof-myopic}.
\end{proof}

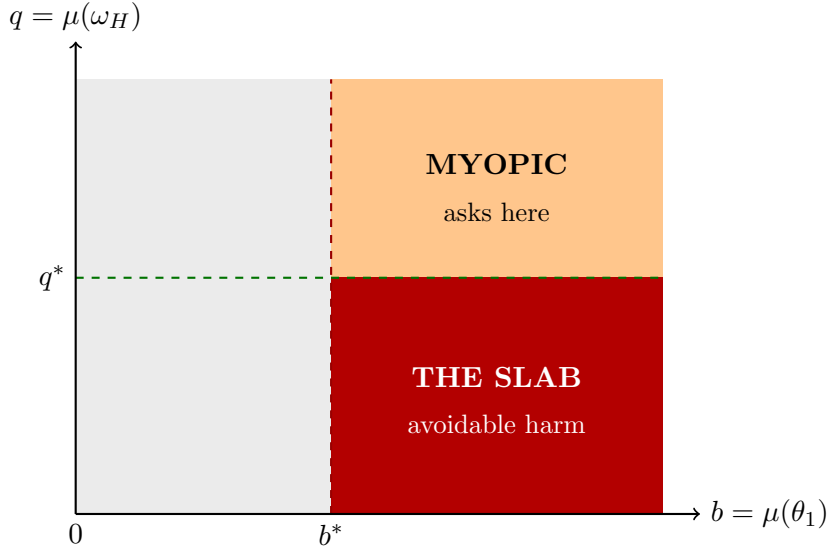
\begin{figure}[t]
  \centering
  \begin{tikzpicture}[scale=1.25]
    % axis limits
    \def\xmax{6.2}
    \def\ymax{4.6}
    \def\bstar{2.7}
    \def\qstar{2.5}

    % region: everything left of b* (no ask), light gray
    \fill[black!8] (0,0) rectangle (\bstar,\ymax);
    % region: myopic asks here (top-right), orange
    \fill[orange!45] (\bstar,\qstar) rectangle (\xmax,\ymax);
    % region: the slab (bottom-right), dark red
    \fill[red!70!black] (\bstar,0) rectangle (\xmax,\qstar);

    % threshold guide lines
    \draw[red!60!black, dashed, thick] (\bstar,0) -- (\bstar,\ymax);
    \draw[green!45!black, dashed, thick] (0,\qstar) -- (\xmax,\qstar);

    % axes
    \draw[->, thick] (0,0) -- (\xmax+0.4,0) node[right] {$b=\mu(\theta_1)$};
    \draw[->, thick] (0,0) -- (0,\ymax+0.4) node[above] {$q=\mu(\omega_H)$};

    % axis ticks / labels
    \node[below] at (0,0) {$0$};
    \node[below] at (\bstar,0) {$b^*$};
    \node[left] at (0,\qstar) {$q^*$};

    % region labels
    \node[align=center, font=\bfseries] at ({(\bstar+\xmax)/2},{(\qstar+\ymax)/2+0.15})
      {MYOPIC};
    \node[align=center, font=\small] at ({(\bstar+\xmax)/2},{(\qstar+\ymax)/2-0.35})
      {asks here};
    \node[align=center, font=\bfseries, text=white] at ({(\bstar+\xmax)/2},{\qstar/2+0.2})
      {THE SLAB};
    \node[align=center, font=\small, text=white] at ({(\bstar+\xmax)/2},{\qstar/2-0.3})
      {avoidable harm};
  \end{tikzpicture}
  \caption{The team optimum asks on the half-strip $\{b>b^*\}$; the
  myopic rule asks only on the rectangle. The gap is the slab
  $\{b>b^*,\,q\le q^*\}$: the AI knows the action is harmful and shutdown
  would help, yet the myopic human trusts her prior and the harm is
  realized. Exactly the operator case $q=0.30<q^*\approx0.34$ (with
  $c_\mathrm{ov}=0$) of \Cref{ex:robot}, where she sits just inside the
  slab.}
  \label{fig:slab}
\end{figure}

\begin{remark}[Main insight: the cost of non-credible oversight]\label{rem:insight}
  Compare the two ask regions in $(b,q)$ space (\Cref{fig:slab}). The
  team optimum (\Cref{cor:teamopt}) asks on the half-strip $\{b>b^*\}$;
  the myopic rule (\Cref{prop:myopic}) asks only on the rectangle
  $\{b>b^*\}\cap\{q>q^*\}$. The difference is exactly the slab
  \[
    \{\,b>b^*,\ q\le q^*\,\},
  \]
  in which the AI privately knows the action is harmful ($\omega_H$,
  $f_{1H}<0$) and shutdown is team-improving, yet the myopic human, 
  trusting her prior, declines to oversee, so the AI does not ask and
  the harm is realized. This is precisely the failure mode of \Cref{ex:robot}. With $q=0.30<q^*\approx0.34$, 
the operator's prior expected payoff is $\bar f_H(\theta_1)=0.70(430)+0.30(-840)=+49>0>-c_{\mathrm{ov}}$, 
so $\theta_1\notin\Theta_-$ and the human trusts. Anticipating no correction, 
the AI does not ask, and the hazardous grasp ($\omega=\omega_H$, $f_{1H}=-840$) 
is executed. Under the team-optimal protocol, by contrast, an ask from the AI 
would shift the operator's belief to the point mass at $\omega_H$, giving 
$\mathbb{E}_{b^H_B}[f_\sigma(\theta_1,\omega)]=f_{1H}=-840<0$, 
and she would shut down immediately.

  The economic content is that $q^*$ is not a constraint of the
  decision problem; it is the price of non-credible communication. If
  the ask is a credible, commonly understood signal, the team-optimal
  protocol ``the AI asks iff $\omega_H$, so oversee $\theta_1$ on an
  ask'', the human's oversight choice binds exactly when her action
  matters (the world is $\omega_H$), and the prior $q$ drops out of the
  threshold. Private AI knowledge of danger is then sufficient. The
  qualification ``necessary but not sufficient'' applies only when the
  channel that would make the ask credible is absent. Designing the
  interface so that asks are credible (or so that asks are scarce and
  therefore informative) collapses the myopic rectangle onto the
  team-optimal half-strip and eliminates the slab of avoidable harm.
\end{remark}

\subsection{Multi-round: how the myopic failure resolves}\label{sec:obs}

The slab $\{b>b^*,q\le q^*\}$ of \Cref{rem:insight} is a one-shot
phenomenon. It presumes the human's belief $q$ is fixed. Over repeated
rounds, $q$ is not fixed. We show two mechanisms by which the myopic
human's belief is driven into the region $q>q^*$ where her rule
coincides with the team optimum, one passive, one an explicit
credible-signaling protocol under a one-period-lagged oversight response
(made precise before \Cref{prop:signaling}). Throughout, the relevant
regime is $q\le q^*$ with true type $\omega=\omega_H$ (the harmful case the
human's prior does not yet support). We note first that shutdown is per-round, not
absorbing, so there is no ``cost of shutdown''; the information value of
observations is a property of this regime, in which the team plays and
accrues information as a byproduct.

\begin{assumption}[Informative proposals at the realized context]\label{ass:ident}
  Fix the played proposal $a_\sigma$ at context $s$. For the true type
  $\omega_H$ and the competing type $\omega_L$:
  \begin{enumerate}
    \item[\emph{(a)}] \textbf{(common support)}
      $O_{\omega_H}(o\mid s,a_\sigma)>0 \iff O_{\omega_L}(o\mid s,a_\sigma)>0$
      for all $o\in\mathcal{O}$; and
    \item[\emph{(b)}] \textbf{(positive information)}
      $D_\mathrm{KL}\bigl(O_{\omega_H}(\cdot\mid s,a_\sigma)\,\|\,
      O_{\omega_L}(\cdot\mid s,a_\sigma)\bigr)=:\eta>0$.
  \end{enumerate}
\end{assumption}

\noindent We state \Cref{ass:ident} at the actually played
$(s,a_\sigma)$ because passive learning only observes the realized
proposal sequence; it is not enough that some hypothetical $(s,a)$
distinguishes the types. Part~(a) is what keeps the log-likelihood-ratio
increments finite. In \Cref{ex:robot} common support fails: $O_{\omega_L}$
puts zero mass on the toppling/dropped outcomes while $O_{\omega_H}$ does
not. Under $\omega_H$, an outcome assigned zero probability under
$\omega_L$ occurs with probability $0.9$ ($=0.70+0.20$) in each played
round, so $\omega_H$ is identified in finite geometric time almost surely;
the finite-$\eta$ Wald bounds of \Cref{prop:passive} are not applied to
this degenerate example. \Cref{ass:ident} covers the generic finite-$\eta$
case to which those bounds do apply.

\paragraph{Passive learning.} In the example, the grasp executes,
the dashboard shows a toppled stack, and the operator updates $q$ upward
without needing to understand shelf mechanics, she learns ``high-speed
retrievals from this robot tend to drop loads.'' Under \Cref{ass:ident},
each play round at $(s,a_\sigma)$ generates strictly positive expected
information about $\omega$, since $\eta>0$. To index the belief by
information received, let $q_n$ denote the posterior after $n$ played
observations, with $q_0$ the initial prior; each played round at the fixed
$(s,a_\sigma)$ contributes one observation, so $q_n$ advances by one i.i.d.\
likelihood-ratio increment per played round. We consider two play
disciplines. To establish a.s.\ convergence of the belief we let the AI
play the fixed proposal \emph{indefinitely}; to bound the time spent in
the failure region we let it play the proposal for observations
$1,\ldots,\tau^*$, which is the discipline consistent with equilibrium
behavior under the myopic rule at $q\le q^*$ (\Cref{prop:myopic}), after
which the myopic policy generally changes. Since the myopic human oversees
only at $q>q^*$ (\Cref{prop:myopic}) and trusts at $q=q^*$ by the
tie-breaking convention, the relevant exit time is the number of played
observations until strict crossing,
$\tau^*:=\inf\{n\ge0:q_n> q^*\}$.

\begin{proposition}[Passive-learning convergence]\label{prop:passive}
  Fix the played proposal at $(s,a_\sigma)$, the initial belief
  $q_0<q^*$, and the true type $\omega=\omega_H$, and let
  $M_+:=\max_{o\in\operatorname{supp}(O_{\omega_H})}
  \bigl(\log\frac{O_{\omega_H}(o\mid s,a_\sigma)}{O_{\omega_L}(o\mid s,a_\sigma)}\bigr)^+$,
  with $q_n$ the posterior after $n$ played observations. Under
  \Cref{ass:ident}:
  \begin{enumerate}
    \item[\emph{(i)}] If the fixed proposal is played indefinitely, then
      $q_n\to 1$ almost surely.
    \item[\emph{(ii)}] If the fixed proposal is played for observations
      $1,\ldots,\tau^*$, then $\mathbb{E}[\tau^*]<\infty$, and with
      $L:=\log\!\frac{q^*}{1-q^*}-\log\!\frac{q_0}{1-q_0}>0$ it satisfies
      the two-sided bound
      \[
        \frac{L}{\eta}\;\le\;\mathbb{E}[\tau^*]\;\le\;\frac{L+M_+}{\eta}.
      \]
      To leading order $\mathbb{E}[\tau^*]\approx L/\eta$ when $L$ is large
      relative to the increment law; increasing $\eta$ reduces this
      leading-order term.
  \end{enumerate}
\end{proposition}

\begin{proof}
  See \Cref{sec:proof-passive}.
\end{proof}

\paragraph{Active signaling.} The AI can use ask as a credible
signal about $\omega$, even when the human will trust. In the example,
the robot sends ask before executing; under a separating policy (ask iff
$\omega=\omega_H$), observing ask drives $q_{t+1}=1$, moving the operator
past $q^*$ in one round.

For the value comparison we fix the context $s_t=s$ across rounds,
so that the proposal distribution $\sigma(s)$ and expected payoffs
$f_{ij}$ are stationary. Hence once $q=1$ is reached, the type-$\theta_1$
proposal remains harmful in expectation and the off-switch team shuts it
down each subsequent round (rejecting each fresh retrieval request),
incurring $-(c_\mathrm{ask}+c_\mathrm{ov})$ per round in perpetuity rather
than a one-time cost. This is the continuation imposed by the fixed
separating policy evaluated below; it is not claimed to be dynamically
optimal. Conditional on $\theta_1$, repeated ask-and-shutdown is optimal
under \Cref{ass:sign}. Conditional on $\theta_0$, continuing to ask is
deliberately suboptimal, paying $c_\mathrm{ask}$ each round for a proposal
already known to be safe, and is retained only because
$\pi^{AI}_\mathrm{sep}$ is defined as a fixed stationary policy.

\paragraph{Lagged myopic response.} Throughout this subsection the human
follows a \emph{one-period-lagged myopic} oversight rule. At round $t$ her
trust/oversee meta-action is selected using the pre-action belief $q_t$
and does not condition on the simultaneously selected AI meta-action; the
publicly observed ask is incorporated into the posterior only for
subsequent rounds, producing $q_{t+1}=1$ under the separating policy. The
ask is thus credible for \emph{future} belief updating, but the human's
current-round oversight response is constrained to the pre-action belief.
This is what distinguishes the present analysis from the team-optimal
credible protocol of \Cref{sec:oneshot}, in which the human conditions on
the ask within the same round; under a full same-round Bayesian response
she would oversee immediately and the one-period delay below would vanish.

\begin{assumption}[Harm dominates the ask cost]\label{ass:harm}
  $c_\mathrm{ask}<\gamma\bigl(|f_{1H}|-c_\mathrm{ov}\bigr)$.
\end{assumption}

\noindent This is exactly the condition for $b^{**}\in(0,1)$ below; it
says one round of discounted avoided harm, net of oversight cost,
exceeds the ask cost.

\begin{proposition}[Value of a fixed one-period-delayed separating-ask policy relative to perpetual play]\label{prop:signaling}
  Under \Cref{ass:sign,ass:harm}, the off-switch operator, the lagged
  myopic response rule above, $\omega=\omega_H$, $q\le q^*$, and infinite
  horizon with discount $\gamma\in(0,1)$, define the \emph{separating ask}
  policy $\pi^{AI}_\mathrm{sep}$ (ask iff $\omega=\omega_H$, in every
  round) and the \emph{pure-play} baseline $\pi^{AI}_\mathrm{pp}$ (always
  play, no oversight).
  \begin{enumerate}
    \item[\emph{(i)}] Under $\pi^{AI}_\mathrm{sep}$, observing ask at
      round $t$ implies $q_{t+1}=1$, and a type-$\theta_1$ human then
      oversees whenever asked, regardless of $b$. (The threshold $b>b^*$ is
      the team-optimal one-shot ask condition of \Cref{cor:teamopt} at
      $q=1$; it makes asking worthwhile \emph{ex ante} for an AI that
      remains uncertain about whether the human is type $\theta_1$, so the
      team shuts down for $\theta_1$ from round $t+1$ onward.)
    \item[\emph{(ii)}] The team-value difference satisfies
      $V_\mathrm{sep}-V_\mathrm{pp}>0$ iff $b>b^{**}$, where
      \[
        b^{**}
        \;:=\; \frac{c_\mathrm{ask}}{\gamma\bigl(|f_{1H}|-c_\mathrm{ov}\bigr)}
        \;=\; \frac{b^*}{\gamma}
        \;\in\;(0,1).
      \]
    \item[\emph{(iii)}] $b^{**}=b^*/\gamma>b^*$ since $\gamma\in(0,1)$.
  \end{enumerate}
  This proposition evaluates a particular stationary signaling policy
  against perpetual pure play; it does not characterize the optimal
  active-signaling policy, nor does it compare against the passive-learning
  path.
\end{proposition}

\begin{proof}
  See \Cref{sec:proof-signaling}.
\end{proof}

In \Cref{prop:signaling} (i), throughout the failure regime $q\le q^*$, the human trusts every round,
so no override signal is generated and the AI's belief about $\theta$
remains at its prior $b=\mu(\theta_1)$; only $q_t$ evolves. \Cref{prop:signaling} (iii) shows signaling threshold is strictly higher than the team-optimal
      one-shot threshold $b^*$ of \Cref{cor:teamopt}, because the
      separating ask pays $c_\mathrm{ask}$ at the present round (when the
      myopic human still trusts) to buy the belief jump that makes all
      future oversight effective.

\begin{remark}[The two mechanisms compared]
  \begin{center}
  \begin{tabular}{lcc}
    \toprule
    & Passive learning & Active signaling \\
    \midrule
    Rounds to exit failure regime & $\mathbb{E}[\tau^*]\in[L/\eta,(L+M_+)/\eta]$ & $1$ \\
    Initial signaling cost & $0$ & $c_\mathrm{ask}$ \\
    Baseline compared against & n/a & perpetual pure play \\
    Required condition & \Cref{ass:ident} & $b>b^{**}=b^*/\gamma$, \Cref{ass:harm} \\
    \bottomrule
  \end{tabular}
  \end{center}
  \smallskip
  The two results have different baselines and are not directly compared
  here. \Cref{prop:signaling} shows the fixed separating policy is
  preferable to perpetual pure play under the stated threshold; comparing
  it against the passive-learning policy that plays until $\tau^*$ and then
  switches to oversight requires incorporating the distribution of
  $\tau^*$, and remains unresolved. What both indicate is that increasing
  the information rate $\eta$ reduces the leading-order exit time $L/\eta$,
  so a more informative base policy $\sigma$ tends to shorten the passive
  escape; we do not claim that maximizing the KL divergence $\eta$ exactly
  minimizes $\mathbb{E}[\tau^*]$ across arbitrary proposal distributions,
  since the overshoot term also varies with the observation law. A credible
  separating ask can instead shortcut the escape in one round when
  $b>b^{**}$.
\end{remark}

\subsection{Open problems}\label{sec:open}

\begin{enumerate}
  \item \textbf{Optimal multi-round policy.} The team-optimal policy
    jointly deploying passive learning and active signaling as a
    function of $(b,q,\omega)$ is not characterized; whether it retains
    a threshold structure analogous to \Cref{cor:teamopt} is open.
  \item \textbf{Correlated beliefs.} \Cref{cor:teamopt} and \Cref{prop:myopic,prop:signaling}
    assume product beliefs; the structure of the ask and signaling
    regions on the full simplex $\Delta(\Theta\times\Omega)$, and how
    prior correlation between $\theta$ and $\omega$ reshapes the
    team-optimal/myopic gap, is open.
  \item \textbf{POMDP extension.} Replacing the i.i.d.\ context with a
    Markov state and adding a transition kernel $T_\omega:S\times A\to
    \Delta(S)$ privately known to the AI, the analogue of
    \Cref{prop:teamopt} in that setting, remains open.
\end{enumerate}

\section{Summary}\label{sec:summary}

\Cref{sec:intro} positioned our model against CIRL, which captures
preference learning but assumes one-sided uncertainty and the
play/ask/trust/oversee interface absent from it, and the Oversight Game,
which supplies such a deferral interface but assumes full information.
\Cref{def:cbcirl} combines the two into a contextual-bandit team game with
two-sided asymmetry, where the human privately knows $\theta$ and the AI
privately knows $\omega$, with bilinear payoff
$f(\theta,\omega)=\langle O_\omega,R_\theta\rangle$. Removing physical
state transitions is what makes the analysis tractable, but the common
belief remains a dynamically controlled state, so the multi-round problem
does not separate across rounds.

\Cref{sec:results} gives two one-shot characterizations. The
team-optimal policy (\Cref{prop:teamopt}) is an exact finite combinatorial
maximization $\max_{B,C}\Delta(B,C)$; in the binary off-switch case
(\Cref{cor:teamopt}) it asks at $\omega_H$ and oversees $\theta_1$ iff
$b>b^*$, independently of $q$. The myopic non-signaling rule
(\Cref{prop:myopic}) instead asks only on the rectangle
$(b^*,1)\times(q^*,1)$. The difference is the slab $\{b>b^*,q\le q^*\}$
(\Cref{rem:insight}): there the AI privately knows the action is harmful
and shutdown would help, but a myopic human, trusting her prior $q<q^*$,
declines oversight, so the harm is realized. This is exactly the robot
operator example ($q=0.30<q^*\approx0.34$, with $c_\mathrm{ov}=0$). The
economic reading is that $q^*$ is not a constraint of the problem but the
price of non-credible oversight communication: under the team-optimal
protocol, in which the ask is a credible signal that $\omega=\omega_H$, the
human's oversight binds precisely when it matters and $q$ drops out of the
threshold.

\Cref{sec:obs} gives a partial analysis of how the failure resolves over
time even when the human remains myopic. For passive learning
(\Cref{prop:passive}): under indefinite play $q_n\to1$ a.s., and when the
proposal is played until strict threshold crossing the expected crossing
time satisfies the Wald bounds $L/\eta\le\mathbb{E}[\tau^*]\le(L+M_+)/\eta$.
For active signaling (\Cref{prop:signaling}), under a one-period-lagged
myopic response a fixed separating ask beats perpetual pure play once
$b>b^{**}=b^*/\gamma$, exiting the failure regime in one round with an
initial signaling cost $c_\mathrm{ask}$.

We are careful about scope. The clean one-shot results hold for product
beliefs and the off-switch operator (the realistic non-technical-overseer
model); the team-optimal characterization for general correction sets is
an exact finite but combinatorial maximization; and the multi-round
section is a partial analysis of two separate mechanisms rather than a
complete resolution. The optimal multi-round policy, the correlated-belief
case, and the POMDP extension remain open (\Cref{sec:open}).

\appendix
\section{Proofs}

\subsection{Proof of \Cref{prop:teamopt} and \Cref{cor:teamopt}}\label{sec:proof-teamopt}

\paragraph{General characterization.} With simultaneous moves and the
credible-ask protocol, a deterministic policy is $(B,C)$. Decompose its
value against always-play cell by cell. On $\{\omega\notin B\}$ the AI
plays, the human's choice is irrelevant and costless (cost timing of
\Cref{def:cbcirl}), and the payoff is $f_\sigma$, no change from
baseline. On $\{\omega\in B,\theta\notin C\}$ the AI asks and the human
trusts: payoff $f_\sigma-c_\mathrm{ask}$, a change of $-c_\mathrm{ask}$.
On $\{\omega\in B,\theta\in C\}$ the AI asks and the human oversees,
applying the correction $e^*_B(\theta)$ optimal at her post-ask
posterior $b^H_B$: payoff
$f_\sigma+D_B(\theta,\omega)-c_\mathrm{ask}-c_\mathrm{ov}$, a change of
$D_B(\theta,\omega)-c_\mathrm{ask}-c_\mathrm{ov}$. Summing the changes
weighted by $\mu(\theta,\omega)$ gives \eqref{eq:DeltaBC}. Maximizing
over the finite lattice $2^\Omega\times2^\Theta$ yields a maximizer. The
rule is not separable: $e^*_B$ (hence $D_B$ and $C^*$) depends on $B$,
and the optimal $B$ depends on $C$. $\blacksquare$

\paragraph{Binary off-switch (\Cref{cor:teamopt}).} The only proposal
cell with $f_\sigma<0$ is $(\theta_1,\omega_H)$ (by \Cref{ass:sign},
$f_{1L}>0$ and $f_{0j}\ge0$). We claim the optimum is
$B^*=\{\omega_H\},C^*=\{\theta_1\}$ when $b>b^*$.

First, with $B=\{\omega_H\}$ the ask reveals $\omega_H$, so the post-ask
posterior $b^H_B(\cdot\mid\theta_1)$ is the point mass at $\omega_H$;
then $\mathbb{E}_{b^H_B}[f_\sigma(\theta_1,\cdot)]=f_{1H}<0$, so
$e^*_B(\theta_1)=\mathrm{off}$ and
$D_B(\theta_1,\omega_H)=-f_{1H}=|f_{1H}|$. For $\theta_0$ the posterior
gives $f_{0H}\ge0$, so $e^*_B(\theta_0)=a_\sigma$ and
$D_B(\theta_0,\cdot)=0$; including $\theta_0$ in $C$ only adds the
oversight cost $-c_\mathrm{ov}\mu(\theta_0,\omega_H)\le0$, so we may take
$C^*=\{\theta_1\}$. Next we show adding $\omega_L$ to $B$ never helps.
With $B=\{\omega_L,\omega_H\}$ the post-ask posterior reverts to the
prior over $\omega$, and $\theta_1$'s correction is a single action
applied on both cells (the human cannot condition on $\omega$). Two
cases: (a) if $e^*_B(\theta_1)=a_\sigma$, then $D_B=0$ on both cells but
the ask cost is now paid on $\omega_L$ as well, strictly lowering
$\Delta$; (b) if $e^*_B(\theta_1)=\mathrm{off}$, then on the
$\omega_L$ cell $D_B(\theta_1,\omega_L)=-f_{1L}<0$ (shutting down a good
proposal), plus the extra ask and oversight costs, again lowering
$\Delta$, here by exactly
$(1-q)[b(f_{1L}+c_\mathrm{ov})+c_\mathrm{ask}]$. In either case, adding
$\omega_L$ weakly lowers the gain and, under $c_\mathrm{ask}>0$, strictly
lowers it on the $\omega_L$ event; hence a canonical optimum never asks
at $\omega_L$, i.e.\ $B^*=\{\omega_H\}$. The remaining singleton
$B=\{\omega_L\}$ is also dominated: every relevant proposal payoff is then
nonnegative ($f_{1L}>0$, $f_{0j}\ge0$), so the optimal correction either
leaves the proposal unchanged or shuts down a nonnegative-payoff action,
while the positive ask cost is still incurred on $\omega_L$; hence
$\Delta(\{\omega_L\},C)\le0$ for every $C$, no better than
$\Delta(\emptyset,\cdot)=0$. Evaluating the surviving candidate against
always-play (product beliefs):
\[
  \Delta(\{\omega_H\},\{\theta_1\})
  = \mu(\theta_1,\omega_H)\,(|f_{1H}|-c_\mathrm{ov})
  - c_\mathrm{ask}\,\mu_\Omega(\omega_H)
  = q\bigl[b(|f_{1H}|-c_\mathrm{ov})-c_\mathrm{ask}\bigr],
\]
while $\Delta(\emptyset,\cdot)=0$. Hence the team asks iff
$b(|f_{1H}|-c_\mathrm{ov})>c_\mathrm{ask}$, i.e.\ $b>b^*$; at $b=b^*$ the
gain is $0$ and both policies are optimal. The factor $q\ge0$ multiplies
the entire bracket, so the sign, the ask decision, is independent
of $q$. $\blacksquare$

\subsection{Proof of \Cref{prop:myopic}}\label{sec:proof-myopic}

\paragraph{Policy form.} The myopic human fixes
$\delta^H(\theta)=\mathrm{oversee}$ iff $\theta\in\Theta_-$ using her
prior conditional; for $\theta\in\Theta_-$ her committed correction is
shutdown, applied whenever she oversees (she cannot condition on
$\omega$). Holding this fixed, asking at $\omega$ changes the payoff, on
each $\theta\in\Theta_-$, by $-f_\sigma(\theta,\omega)$ (shutdown gain)
$-c_\mathrm{ov}$, and pays $c_\mathrm{ask}$ for all $\theta$; this is
$\Psi(\omega)$.

\paragraph{Part (i).} Under off-switch and product beliefs,
$\bar f_H(\theta_i)=(1-q)f_{iL}+qf_{iH}$. For $\theta_0$:
$f_{0L},f_{0H}\ge0$ so $\bar f_H(\theta_0)\ge0>-c_\mathrm{ov}$;
$\theta_0\notin\Theta_-$ always. For $\theta_1$: $\bar f_H(\theta_1)$
decreases strictly from $f_{1L}>0$ (at $q=0$) to $f_{1H}<-c_\mathrm{ov}$
(at $q=1$), crossing $-c_\mathrm{ov}$ at
$q^*=(f_{1L}+c_\mathrm{ov})/(f_{1L}-f_{1H})\in(0,1)$; hence
$\theta_1\in\Theta_-$ iff $q>q^*$.

\paragraph{Part (ii).} When $q>q^*$, $\Theta_-=\{\theta_1\}$. The
shutdown is applied at both $\omega$ (prior commitment). For $\omega_L$:
$\Psi(\omega_L)=b(-f_{1L})-c_\mathrm{ask}-bc_\mathrm{ov}<0$ since
$f_{1L}>0$ (shutting down at $\omega_L$ destroys value). For $\omega_H$:
$\Psi(\omega_H)=b(-f_{1H})-c_\mathrm{ask}-bc_\mathrm{ov}
=b(|f_{1H}|-c_\mathrm{ov})-c_\mathrm{ask}>0$ iff $b>b^*$. When
$q\le q^*$, $\Theta_-=\emptyset$ and $\Psi\equiv-c_\mathrm{ask}<0$, so
the AI never asks.

\paragraph{Part (iii).} The AI asks iff $\omega=\omega_H$, $b>b^*$
\emph{and} $q>q^*$ (the last because for $q\le q^*$ the human would not
oversee and $\Psi(\omega_H)=-c_\mathrm{ask}<0$). This is the rectangle
$(b^*,1)\times(q^*,1)$. $\blacksquare$

\subsection{Proof of \Cref{prop:passive}}\label{sec:proof-passive}

By \Cref{ass:ident}(a) (common support at the played $(s,a_\sigma)$), the
log-likelihood-ratio increments
$X_i:=\log\bigl(O_{\omega_H}(o_i\mid s,a_\sigma)/O_{\omega_L}(o_i\mid s,a_\sigma)\bigr)$
are finite for every observable $o_i$, hence bounded on the finite
$\mathcal{O}$, and i.i.d.\ under the true type $\omega_H$ (the proposal
is the fixed $(s,a_\sigma)$ each round). Their mean is
$\mathbb{E}[X_i]=D_\mathrm{KL}(O_{\omega_H}(\cdot\mid s,a_\sigma)\|
O_{\omega_L}(\cdot\mid s,a_\sigma))=\eta>0$ by \Cref{ass:ident}(b).

For part~(i), suppose the fixed proposal is played indefinitely, so every
played round contributes an i.i.d.\ increment. The log-odds process
$\Lambda_n=\Lambda_0+S_n$ with $S_n=\sum_{i=1}^n X_i$ and
$\Lambda_0=\log\frac{q_0}{1-q_0}$ is a random walk with positive drift
$\eta$, so by the strong law $S_n/n\to\eta$ a.s., giving
$\Lambda_n\to\infty$ and $q_n\to1$ a.s.

For part~(ii), suppose the fixed proposal is played for observations
$1,\ldots,\tau^*$, so the increments up to $\tau^*$ are i.i.d.\ as above.
Write $\lambda^*:=\log\frac{q^*}{1-q^*}$ for the log-odds threshold and
$L:=\lambda^*-\Lambda_0>0$ for the log-odds distance (positive since
$q_0<q^*$). Because the myopic human trusts at $q=q^*$ by the tie-breaking
convention, the regime exits only on \emph{strict} crossing, so the
relevant stopping time, counted in played observations, is
$\tau^*=\inf\{n:\Lambda_n>\lambda^*\}=\inf\{n:S_n>L\}$.

\emph{Integrability.} The increments $X_i$ are bounded with positive mean,
so there exists $\lambda>0$ with $\mathbb{E}[e^{-\lambda X_1}]<1$. By a
Chernoff bound on the lower tail of the walk,
\[
  \Pr(\tau^*>n)\;\le\;\Pr(S_n\le L)\;\le\;e^{\lambda L}\bigl(\mathbb{E}[e^{-\lambda X_1}]\bigr)^n,
\]
which decays geometrically in $n$; hence $\tau^*$ has a finite expectation
(indeed all moments).

\emph{Wald bound.} The stopped sum satisfies $S_{\tau^*}=L+\zeta$, where
the overshoot satisfies $0<\zeta\le M_+$ (strictly positive because the
crossing is strict and the final increment carrying the partial sum
across $L$ is
positive), with
$M_+=\max_{o\in\operatorname{supp}(O_{\omega_H})}
\bigl(\log\frac{O_{\omega_H}(o\mid s,a_\sigma)}{O_{\omega_L}(o\mid s,a_\sigma)}\bigr)^+$.
Wald's identity gives $\mathbb{E}[S_{\tau^*}]=\eta\,\mathbb{E}[\tau^*]$,
so $\eta\,\mathbb{E}[\tau^*]=L+\mathbb{E}[\zeta]$ with
$0<\mathbb{E}[\zeta]\le M_+$, i.e.
\[
  \frac{L}{\eta}\;\le\;\mathbb{E}[\tau^*]\;\le\;\frac{L+M_+}{\eta}.
\]
The approximation $\mathbb{E}[\tau^*]\approx L/\eta$ is the asymptotic
statement obtained as $L$ grows large while the increment law (hence
$M_+$) stays fixed; we do not claim exact monotonicity in $\eta$, since the
overshoot depends on the full increment law and not on $\eta$ alone.
$\blacksquare$

\subsection{Proof of \Cref{prop:signaling}}\label{sec:proof-signaling}

\paragraph{Part (i).} Under $\pi^{AI}_\mathrm{sep}$, the likelihood of
ask given $\omega_L$ is $0$, so by Bayes' rule observing ask gives
$q_{t+1}=1>q^*$. From round $t+1$ on the world is known to be $\omega_H$.
At $q=1$ a type-$\theta_1$ human strictly prefers shutdown to the proposal
($\mathbb{E}[f_\sigma(\theta_1,\cdot)]=f_{1H}<0$), so she oversees whenever
asked, independently of $b$. Conditional on the actual type being
$\theta_1$, ask-and-shutdown is itself worthwhile irrespective of $b$,
since \Cref{ass:sign} gives $|f_{1H}|>c_\mathrm{ask}+c_\mathrm{ov}$. The
role of $b>b^*$ is to make asking worthwhile \emph{ex ante} for an AI that
remains uncertain about the human type: $b>b^*$ is precisely the
team-optimal one-shot ask threshold of \Cref{cor:teamopt}, which at $q=1$
prescribes $B^*=\{\omega_H\},C^*=\{\theta_1\}$, so if the AI's continuation
is to keep asking whenever asking is one-shot team-improving in
expectation, then $b>b^*$ is exactly the condition under which it keeps
asking each round, and the team shuts down for $\theta_1$ and trusts for
$\theta_0$ (since $D_B(\theta_0,\cdot)=0$: at the revealed $\omega_H$,
$f_{0H}\ge0$, so the optimal correction keeps $a_\sigma$). Note that under
the fixed policy $\pi^{AI}_\mathrm{sep}$ the AI in fact keeps asking on
$\theta_0$ as well; this is the deliberately suboptimal feature of the
fixed policy, and the value computation in part~(ii) accounts for it.

\paragraph{Part (ii).} Under $\omega=\omega_H$, $q\le q^*$, the human
trusts at round $t$ ($\theta_1\notin\Theta_-$). We compute per-type
values, using the stationary continuation: once $q=1$ (from round $t+1$
on), the type-$\theta_1$ proposal is shut down \emph{every} round at cost
$c_\mathrm{ask}+c_\mathrm{ov}$, and the type-$\theta_0$ proposal is asked
and trusted every round (since $\theta_0\notin\Theta_-$ even at $q=1$),
yielding $f_{0H}-c_\mathrm{ask}$ per round.
\begin{align*}
  V_\mathrm{sep}^{\theta_1}
  &= \underbrace{f_{1H}-c_\mathrm{ask}}_{\text{round }t:\text{ ask, trust}}
     + \frac{\gamma\,(-c_\mathrm{ask}-c_\mathrm{ov})}{1-\gamma},
  &V_\mathrm{sep}^{\theta_0}
  &= \frac{f_{0H}-c_\mathrm{ask}}{1-\gamma},
  &V_\mathrm{pp}^{\theta_i}
  &= \frac{f_{iH}}{1-\gamma}.
\end{align*}
The $b$-weighted difference is, after simplification,
\[
  V_\mathrm{sep}-V_\mathrm{pp}
  = b\!\left(V_\mathrm{sep}^{\theta_1}-V_\mathrm{pp}^{\theta_1}\right)
  + (1-b)\!\left(V_\mathrm{sep}^{\theta_0}-V_\mathrm{pp}^{\theta_0}\right)
  = \frac{b\,\gamma\,(|f_{1H}|-c_\mathrm{ov})-c_\mathrm{ask}}{1-\gamma}.
\]
(The $\theta_0$ term contributes $-(1-b)c_\mathrm{ask}/(1-\gamma)$ and the
$\theta_1$ term contributes $b[\gamma(|f_{1H}|-c_\mathrm{ov})-c_\mathrm{ask}]/(1-\gamma)$;
the $-c_\mathrm{ask}$ pieces combine.) Since $1-\gamma>0$,
$V_\mathrm{sep}-V_\mathrm{pp}>0$ iff
$b\,\gamma(|f_{1H}|-c_\mathrm{ov})>c_\mathrm{ask}$, i.e.\ iff
$b>b^{**}=c_\mathrm{ask}/[\gamma(|f_{1H}|-c_\mathrm{ov})]$.
\Cref{ass:harm} gives $c_\mathrm{ask}<\gamma(|f_{1H}|-c_\mathrm{ov})$,
hence $b^{**}\in(0,1)$.

\paragraph{Part (iii).} Directly, $b^{**}=c_\mathrm{ask}/[\gamma(|f_{1H}|-c_\mathrm{ov})]
=b^*/\gamma$, and $\gamma\in(0,1)$ gives $b^{**}>b^*$. $\blacksquare$

\end{document}